\PassOptionsToPackage{sort}{natbib}

\documentclass{ecai} 

\usepackage{latexsym}
\usepackage{amssymb}
\usepackage{amsmath}
\usepackage{amsthm}
\usepackage{booktabs}
\usepackage{enumitem}
\usepackage{graphicx}
\usepackage{color}

\usepackage{algorithm}
\usepackage{algorithmic}

\usepackage{amsmath}
\usepackage{booktabs}
\usepackage{amssymb}
\usepackage{multirow}
\DeclareMathOperator*{\argmax}{arg\,max}

\newcommand{\tr}{{{\mathsf T}}}

\newtheorem{proposition}{Proposition}
\newtheorem{remark}{Remark}

\newcommand{\BibTeX}{B\kern-.05em{\sc i\kern-.025em b}\kern-.08em\TeX}

\let\oldequation\equation
\let\oldendequation\endequation
\renewenvironment{equation}
{\linenomathNonumbers\oldequation}
{\oldendequation\endlinenomath}
\let\oldsubequation\subequations
\let\oldendsubequation\endsubequations
\renewenvironment{subequations}
{\linenomathNonumbers\oldsubequation}
{\oldendsubequation\endlinenomath}
\let\oldmultline\multline
\let\oldendmultline\endmultline
\renewenvironment{multline}
{\linenomathNonumbers\oldmultline}
{\oldendmultline\endlinenomath}

\begin{document}


\begin{frontmatter}


\paperid{1462}


\title{Verification of Geometric Robustness of Neural Networks via Piecewise Linear Approximation and Lipschitz Optimisation}


\author[A]{\fnms{Ben}~\snm{Batten}\thanks{Corresponding author. Email: b.batten@imperial.ac.uk.}}
\author[B]{\fnms{Yang}~\snm{Zheng}}
\author[C]{\fnms{Alessandro}~\snm{De Palma}}
\author[E,D]{\fnms{Panagiotis}~\snm{Kouvaros}}
\author[A,E]{\fnms{Alessio}~\snm{Lomuscio}}

\address[A]{Imperial College London, UK}
\address[B]{University of California San Diego, USA}
\address[C]{Inria, École Normale Supérieure, PSL University, CNRS, France}
\address[D]{Department of Information Technologies, University of Limassol, Cyprus}
\address[E]{Safe Intelligence, UK}


\begin{abstract}
    We address the problem of verifying neural networks against
    geometric transformations of the input image, including rotation,
    scaling, shearing, and translation. The proposed method computes
    provably sound piecewise linear constraints for the pixel values
    by using sampling and linear approximations in combination with
    branch-and-bound Lipschitz optimisation. The method
    obtains provably tighter over-approximations of the perturbation
    region than the present state-of-the-art.
We report results from
    experiments on a comprehensive set of verification benchmarks on MNIST and CIFAR10.
We show that
    our proposed implementation resolves up to 32\% more verification cases than
    present approaches.
\end{abstract}

\end{frontmatter}


\section{Introduction}
Neural networks as used in mainstream applications, including computer
vision, are known to be fragile and susceptible to adversarial
attacks~\cite{goodfellow2014explaining}. The area of \emph{formal
verification of neural networks} is concerned with the development of
methods to establish whether a neural network is \emph{robust}, with
respect to its classification output, to variations of the image.
A large body of literature has so far focused on norm-bounded input
perturbations, aiming to demonstrate that imperceptible adversarial alterations of the pixels cannot alter the classifier's
classification ($\ell_p$ robustness). 
In safety-critical applications such as autonomous driving, however, resistance to norm-bounded perturbations is inadequate to guarantee safe deployment. 
In fact, image classifiers need to be
robust against a number of variations of the image, including
contrast, luminosity, hue, and beyond. 
A particularly important class
of specifications concerns robustness to geometric
perturbations of the input
image~\cite{KouvarosLomuscio18,mohapatra2020towards,singh2019abstract,balunovic2019certifying}. These
may include translation, shearing, scaling, and rotation. 

Owing to the highly nonlinear variations of the
pixels in geometric transformations, verifying robustness to these perturbations
is intrinsically a much harder problem than $\ell_p$ robustness.
Previous work over-approximates these variations through hyper-rectangles~\cite{singh2019abstract} or pairs of linear bounds over the pixel values~\cite{balunovic2019certifying}, 
hence failing to capture most of the complexity of the perturbation region.
Developing more precise methods for verifying geometric
robustness remains an open challenge. In this paper we work towards this end. Specifically, we make
three contributions:
\begin{enumerate}
\item
We present a piecewise linear relaxation method to approximate the set
of images generated by geometric transformations, including rotation,
translation, scaling, and shearing.  This construction can incorporate
previous approaches~\cite{singh2019abstract,balunovic2019certifying}
as special cases while supporting additional constraints, allowing
significantly tighter over-approximations of the perturbation region.

\item We show that sound piecewise linear constraints, the building blocks of the proposed relaxation, can be generated 
via suitable modifications of a previous
approach~\citep{balunovic2019certifying} that generates linear
constraints using sampling, linear and Lipschitz optimisation.  We
derive formal results as well as effective heuristics that enable us
to improve the efficiency of the linear and Lipschitz optimisations in
this context
(cf. Propositions~\ref{prop:linear_optimisation}---\ref{proposition:gradient}). As
we demonstrate, the resulting piecewise constraints can be readily
used within existing tight neural network verifiers.

\item We introduce an efficient implementation for the verification method
above and discuss experimental results showing considerable gains in
terms of verification~accuracy on a comprehensive set of benchmark networks.
\end{enumerate}

The rest of this paper is organized as follows: Section~\ref{sec:related-work} discusses related work. In Section~\ref{section:geometric_robustness} we
introduce the problem of verifying neural networks against geometric
robustness properties. In Section~\ref{section:pw_linar_bounds} we present our
novel piecewise linear approximation strategy via sampling,
optimisation and shifting. In Section~\ref{section:experiments} we discuss
the experimental results obtained and contrast the present method
against the state-of-the-art on benchmark networks. We
conclude in Section~\ref{section:conclusions}. Our code is publicly available on GitHub\footnote{\url{https://github.com/benbatten/PWL-Geometric-Verification}}.

\section{Related Work} \label{sec:related-work}

We here briefly discuss related work from $l_p$-based
 neural network verification, geometric robustness and
 formal verification thereof.

\paragraph{$\ell_p$ robustness verification.}
There is a rich body of work on the verification of neural networks against $\ell_p$-bounded perturbations:
see, e.g.,~\cite{Liu+20a} for a survey. 
Neural network verifiers typically rely on Mixed-Integer Linear Programming (MILP)~\cite{botoeva2020efficient,tjeng2017evaluating}, branch-and-bound~\citep{bunel2020branch,BunelDP20,DePalma2021,xu2021fast,betacrown,HenriksenLomuscio21,Ferrari2022,zhang2022gcpcrown},
 or on abstract interpretation~\cite{Gehr2018ai2, singh2018fast, singh2019anabstract}.
These methods cannot be
used to certify geometric robustness out of the box, as $\ell_p$ balls are unable to accurately
represent geometric transformations~\cite{KouvarosLomuscio18,singh2019abstract}.

\paragraph{Geometric robustness.} The vulnerability of neural networks to
geometric transformations has been observed in~\cite{pmlr-v97-engstrom19a,Fawzi}. A common theme among these works is their
\textit{quantitative} nature, whereby measures of invariance to geometric
robustness are discussed~\cite{Kanbak_2018_CVPR} and
methods to improve spatial robustness are developed. These are based
on
augmentation~\cite{Xiang20,Xiao}, regularisation schemes~\cite{Yang19}, robust optimisation~\cite{pmlr-v97-engstrom19a} and
specialised, invariance-inducing network
architectures~\cite{Jaderberg}. Differently from the cited works, our
key aim here is the \textit{qualitative analysis }of networks towards
establishing formal guarantees of geometric robustness.

\paragraph{Formal verification of geometric robustness.}
One of the earliest works~\cite{Pei} on this subject discretises the
transformation domains, enabling robustness verification through the
evaluation of the model at a finite number of discretised
transformations.
In contrast to~\citet{Pei}, we here focus on continuous
domains, which do not allow exhaustive evaluation. 
Previous work on continuous domains relies on over-approximations, whereby, for each pixel,
the set of allowed values under the perturbation is replaced by a convex relaxation~\cite{KouvarosLomuscio18,singh2019abstract,balunovic2019certifying}.
In particular,~\cite{KouvarosLomuscio18}
and~\cite{singh2019abstract} use an $l_\infty$ norm ball and intervals
respectively, resulting in loose over-approximations.
\citet{balunovic2019certifying} devise more precise convex relaxations
by computing linear approximations with respect to the transformation parameters. 
In this work we further improve precision over \citet{balunovic2019certifying}
by deriving piecewise linear approximations. 
While the above works consider the geometric transformation as a whole (see Section~\ref{section:geometric_robustness}),~\citet{mohapatra2020towards} 
decompose the transformation into network layers to be pre-pended to the network under analysis, resulting in looser approximations when using standard neural network verifiers.
More recently, randomised smoothing techniques have been investigated for geometric robustness~\cite{fischer2020certified,li2021tss,fischer2021scalable}: differently from our work, these only provide probabilistic certificates.
Finally,~\citet{yang2023provable} recently presented a method to train networks more amenable to geometric robustness verification. Our work is agnostic to the training scheme: we here focus on the more challenging general case.



\section{Geometric robustness verification}
\label{section:geometric_robustness}


Our main contribution is a new piecewise linear relaxation of
geometric transformations to verify robustness of neural networks to geometric perturbations. 
We here introduce relevant notation in the
verification problem and present the geometric attack model.

\paragraph{Notation.} Given two vectors $a, b \in \mathbb{R}^n$, we use
$a \geq b$ and $a \leq b$ to represent element-wise inequalities.
Given a vector $a \in \mathbb{R}^m$ and a matrix $A \in \mathbb{R}^{m \times n}$, we denote their elements using
$a[i]$ and $A[i,j]$, respectively.

\paragraph{Neural networks for classification.} We consider a feedforward
neural network with $L$ hidden layers $f: \mathbb{R}^n \rightarrow
\mathbb{R}^m$. Let $x_0 \in \mathbb{R}^n$ denote the input and $x_i$ denotes the activation vectors at layer $i$. We use
$\mathbb{L}_i$ to denote an affine map at layer $i$, \textit{e.g.}, linear,
convolutional, and average pooling operations. Let $\sigma_i$ be an
element-wise activation function, such as ReLU, sigmoid or tanh. The
activation vectors are related by $x_{i+1} =
\sigma_i\big(\mathbb{L}_i(x_i)\big), i = 0, 1, \ldots, L-1$. We are
interested in neural networks for classification: the network output
$f(x_0) = \mathbb{L}_L(x_L)\in \mathbb{R}^m$ represents the score of
each class, and the label $i^*$ assigned to the input $x_0$ is the
class with highest score, \emph{i.e.}, $i^* =\argmax_{i=1, \ldots m} f(x_0)[i]$.

\paragraph{Robustness verification.} 
Let $\mathcal{A}$ be a general
attacker that takes a nominal input $\bar{x} \in \mathbb{R}^n$ and
returns a perturbed input $\mathcal{A}(\bar{x}) \in \mathbb{R}^n$. We
denote the attack space as $\Omega_{\epsilon}(\bar{x}) \subset
\mathbb{R}^n$, \textit{i.e.}, $\mathcal{A}(\bar{x}) \in
\Omega_{\epsilon}(\bar{x})$, where ${\epsilon} > 0$ denotes the attack
budget. 
Formally verifying that a classification neural network $f$ is \emph{robust} with respect to an input $\bar{x}$ and its attack space $\Omega_\epsilon(\bar{x})$ 
implies ensuring that all points in $\Omega_\epsilon(\bar{x})$ will share the same classification label of $\bar{x}$.
This can be done by solving the following optimisation problem $\forall \ i \neq i^*$:
\begin{subequations} \label{eq:nn_verification}
\begin{align}
    \gamma^*_i := \min_{x_0, x_1, \ldots x_L,y} &\quad  y[i^*] - y[i] \nonumber \\
    \text{subject to} & \quad x_0 \in \Omega_{\epsilon}(\bar{x}),
    \label{eq:verification_input_constraints} \\
    & \quad  x_{i+1} = \sigma_i\big(\mathbb{L}_i(x_i)\big), i \in [L]  \label{eq:verification_nn_constraints}\\
     & \quad  y = \mathbb{L}_L(x_L), \label{eq:verification_nn_output}
\end{align}
\end{subequations}
with~\eqref{eq:verification_nn_constraints} being neural network constraints,~\eqref{eq:verification_nn_output} as the neural network output,~\eqref{eq:verification_input_constraints} denoting the attack model constraint, and $ L := \{0, 1, \ldots, L-1\}$.
If $\gamma^*_i > 0\ \forall \ i \ \in \ \{1,\dots,m\}$, the network is certified to be robust.

Even when $\Omega_\epsilon(\bar{x})$ is a convex set, such as in the case of $\ell_p$ perturbations, for which $\Omega_{\epsilon}(\bar{x}) = \{x \in \mathbb{R}^n \mid \|x -
\bar{x}\|_p \leq \epsilon \}$,
the nonconvex neural network
constraints~\eqref{eq:verification_nn_constraints}
make the verification problem~\eqref{eq:nn_verification} difficult to solve. 
However, in this setting, tractable lower bounds  $\underline{\gamma}^*_i \leq \gamma^*_i$ on the solution can be obtained through a variety of techniques, including: 
linear relaxations~\citep{Ehlers17,singh2019abstract,Tran20,tjandraatmadja+2020convex,Zhang2018,Salman2019}, semi-definite programming~\citep{Dathathri+2020a,Batten+21a,RaghunathanSteinhardtLiang2018,Fazlyab2020sdp} and Lagrangian duality~\citep{WongKolter17,BunelDP20,DePalma2021,Dvijotham2018,zhang2022gcpcrown}.
These techniques lie at the core of the network verifiers described in Section~\ref{sec:related-work}.
If $\underline{\gamma}^*_i > 0\ \forall \ i \ \in \ \{1,\dots,m\}$, the network is robust, but a negative lower bound will leave the property undecided, pointing to the importance of tight lower bounds.
When considering geometric transformations, the attack model constraint~\eqref{eq:verification_input_constraints} is highly nonconvex, making verification even more challenging.
%

\paragraph{Attack model via geometric transformation.} A geometric
transformation of an image is a composite function, consisting of a spatial transformation
$\mathcal{T}_\mu$, a bilinear interpolation $\mathcal{I}(u,v)$, which
handles pixels that are mapped to non-integer coordinates, and
changes in brightness and contrast $\mathcal{P}_{\alpha, \beta}$. The
spatial transformation $\mathcal{T}_{\mu}$ can be a composition of
rotation, translation, shearing, and scaling; see \textit{e.g.},~\cite{balunovic2019certifying} for detailed descriptions. The pixel value $\hat{p}_{u,v}$ at position $(u,v)$ of
the transformed image is obtained as follows: (1) the pre-image of
$(u,v)$ is calculated under $\mathcal{T}_\mu$; (2) the resulting
coordinate is interpolated via $\mathcal{I}$ to obtain a value $\xi$;
(3) $\mathcal{P}_{\alpha, \beta}(\xi) = \alpha \xi + \beta$ is applied  to
compute the final pixel value $\hat{p}_{u,v}$. In other words, we
have that $\hat{p}_{u,v} = \mathcal{G}_{u,v}(\alpha,\beta,\mu)$,
where:
\begin{equation}
	\mathcal{G}_{u,v}(\alpha,\beta,\mu) := \mathcal{P}_{\alpha, \beta} \circ \mathcal{I} \circ \mathcal{T}^{-1}_\mu(u,v).
\end{equation}
%
We consider the following standard bilinear interpolation:
\begin{equation*}
    \mathcal{I}(u,v) = \hspace{-10pt}\sum_{\delta_i, \delta_j \in \{0,1\}}\hspace{-7pt}\! p_{i+\delta_i, j + \delta_j}(1 - |i + \delta_i - u|) 
    (1 - |j + \delta_j - v|),
\end{equation*}
where $(i,j)$ denotes the lower-left corner of the interpolation
region $[i,i+1] \times [j, j+1]$ that contains pixel $(u,v)$, and the
matrix $p$ denotes the pixel values of the original image. Note that
the interpolation function $\mathcal{I}$ is continuous on
$\mathbb{R}^2$ but can be nonsmooth on the boundaries of interpolation
regions. Thus, $\mathcal{G}_{u,v}(\alpha,\beta,\mu)$ is in general
\textit{nonsmooth} with respect to the spatial parameter $\mu$ (\textit{e.g.},
rotation).

For simplicity, in the following we will denote the transformation parameter as $\kappa =(\alpha, \beta, \mu)$.
The geometric attack model assumes interval constraints on $(\alpha, \beta, \mu)$, denoted by $\mathcal{B} \subset \mathbb{R}^d$, where $d$ is the dimension of $\kappa$.
The attack space $\Omega_{\epsilon}(\bar{x})$ from~\eqref{eq:verification_input_constraints} is then defined as the set of all images resulting by the application of
$\mathcal{G}_{u,v}(\kappa)$ on each pixel $(u,v)$ of $\bar{x}$, for all $\kappa \in \mathcal{B}$.
More formally, given $\texttt{im} : \mathbb{R}^{n} \rightarrow \mathbb{R}^{h \times w} $, a mapping re-arranging images into their spatial dimensions, 
$\Omega_{\epsilon}(\bar{x}) = \{x' \in \mathbb{R}^{n} \ |\ \texttt{im}(x')[u, v] \in \Omega_{\epsilon}(\bar{x})[u, v]\}$, with 
$\Omega_{\epsilon}(\bar{x})[u, v] = \{\mathcal{G}_{u,v}(\kappa) | \ \forall \ \kappa \in \mathcal{B}\}$.

\paragraph{Problem statement.} The geometric attack model $\Omega_{\epsilon}(\bar{x})$ defines a highly nonconvex constraint on the admissible image inputs, which is not readily supported by bounding techniques designed for $\ell_p$ perturbations. 
As a result, previous work replaces it by
over-approximations~\cite{balunovic2019certifying,singh2019abstract},
which allow verification through $\ell_p$-based neural network
verifiers.  Nevertheless, as described in
Section~\ref{sec:related-work}, their over-approximations are
imprecise, resulting in loose lower bounds $\underline{\gamma}^*_i$.
In this work, we aim to derive a tighter convex relaxation of the
geometric attack model $\Omega_{\epsilon}(\bar{x})$ based on piecewise
linear constraints.  By relying on networks verifiers with support for
these constraints, we will then show that our approach leads to
effective verification bounds for~\eqref{eq:nn_verification}.


\section{Piecewise linear formulation} \label{section:pw_linar_bounds}


As mentioned above, the pixel value function
$\mathcal{G}_{u,v}(\kappa)$ at location $(u,v)$ is generally
\textit{nonlinear} and \textit{nonsmooth} with respect to the transformation parameters
$\kappa$. This is one source of difficulty for solving the
verification problem~\eqref{eq:nn_verification}. In this section, we
introduce a new convex relaxation method to derive tight
over-approximations of $\mathcal{G}_{u,v}(\kappa)$. 

\begin{figure}[t]
	\centering
    \includegraphics[width=\columnwidth]{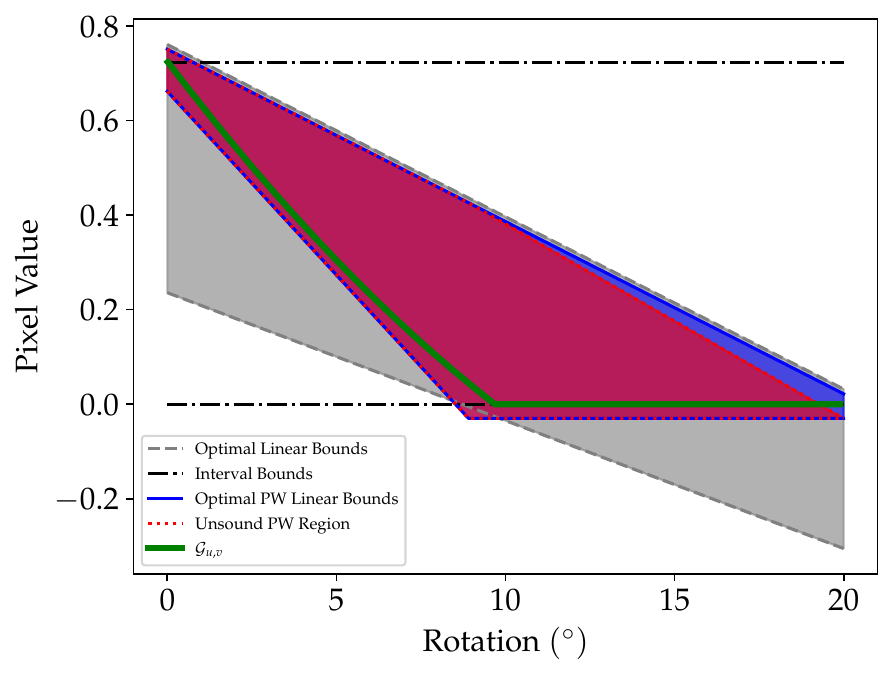}
    \caption{Comparison of sound and unsound piecewise (PW) linear domains (our
    work), sound linear domain (gray area)~\protect \cite{singh2019abstract}, and
    interval bounds (dashed line)~\protect \cite{singh2019abstract}. The true pixel
    value function (the green curve) is marked for a rotation of 18$^{\circ}$.}
\label{fig:linear_pw_comparison}
\vspace{5pt}
\end{figure}

\subsection{Piecewise linear bounds}

Deriving an interval bound for each pixel $(u,v)$, \textit{i.e.}, $L_{u,v} \leq
\mathcal{G}_{u,v}(\kappa) \leq U_{u,v}$, for all $\kappa \in
\mathcal{B}$ and lower and upper bounds $L_{u,v}, U_{u,v} \in \mathbb{R}$, is  arguably the simplest way to get a convex relaxation~\cite{singh2019abstract,KouvarosLomuscio18}. However, even a small geometric transformation can lead to a large interval bound, making this approach too loose for effective verification.

This naive interval bound approach has been extended in~\cite{balunovic2019certifying}, where linear lower and upper bounds were used for each pixel value, \textit{i.e.},
\begin{equation} \label{eq:linear_bounds}
   \underline{w}^\tr \kappa + \underline{b} \leq
   \mathcal{G}_{u,v}(\kappa) \leq \overline{w}^\tr \kappa + \overline{b}, \quad \forall \kappa \in \mathcal{B}.
\end{equation}
The linear bounds~\eqref{eq:linear_bounds}, however, can be still too
loose to approximate the nonlinear function
$\mathcal{G}_{u,v}(\kappa)$ (see Figure~\ref{fig:linear_pw_comparison} for illustration). Our key
idea is to use piecewise linear bounds to approximate the pixel
values:
\begin{equation} \label{eq:pw_linear_bounds}
   \max_{j = 1,\ldots,q} \{\underline{w}_j^\tr \kappa +
   \underline{b}_j\} \leq \mathcal{G}_{u,v}(\kappa) \leq \min_{j =
   1,\ldots,q} \{\overline{w}_j^\tr \kappa + \overline{b}_j\},
\end{equation}
$\forall \kappa \in \mathcal{B}$, where $q$ is the number of piecewise segments, $\underline{w}_j \in
\mathbb{R}^d, \underline{b}_j \in \mathbb{R}, j = 1, \ldots, q$ define
the piecewise linear lower bound, and $\overline{w}_j\in \mathbb{R}^d,
\overline{b}_j \in \mathbb{R}, j = 1, \ldots, q$ define the piecewise
linear upper bound. 
We remark that the pixel values constrained by~\eqref{eq:pw_linear_bounds} form a convex set.
Furthermore, our approach can include the strategies in~\cite{singh2019abstract,balunovic2019certifying} as special cases. 
Employing the relative constraints among the piecewise segments will result in a tighter set. 

For each pixel value, we would like to derive optimal and sound piecewise linear bounds by minimizing the approximation error. Specifically, we aim to compute the lower bound via
\begin{equation}\label{eq:bound_integral_optimisation}
    \begin{aligned}
        \min_{\underline{w}_j, \underline{b}_j, j = 1, \ldots, q}  \quad & \int_{\mathcal{B}} \left(\mathcal{G}_{u,v}(\kappa) - \big(\max_{j = 1,\ldots,q} \{\underline{w}_j^\tr \kappa + \underline{b}_j\}\big)\right) d\kappa \\
        \text{s.t.} \quad &  \max_{j = 1,\ldots,q} \{\underline{w}_j^\tr \kappa + \underline{b}_j\} \leq \mathcal{G}_{u,v}(\kappa), \quad \forall \kappa \in \mathcal{B}.
    \end{aligned}
\end{equation}
Computing the upper bound for~\eqref{eq:pw_linear_bounds} is similar. This optimisation
problem~\eqref{eq:bound_integral_optimisation} is highly nontrivial to solve
since 
\textit{the integral cost function is hard to evaluate due to the nonlinearity of} $\mathcal{G}_{u,v}(\kappa)$.
Motivated by~\cite{balunovic2019certifying}, we first
sample the transformation parameter $\kappa_i$ from $\mathcal{B}$ to
obtain the sampled pixel values $\mathcal{G}_{u,v}(\kappa_i)$, and
then solve a sampled version
of~\eqref{eq:bound_integral_optimisation}. The resulting piecewise
bound is guaranteed to be sound on the sampling points $\kappa_i \in
\mathcal{B}$  but could be unsound on non-sampled points. To derive a
final sound piecewise bounds for  $\mathcal{G}_{u,v}(\kappa)$,
we bound the maximum violation over the entire
$\mathcal{B}$ using a branch-and-bound Lipschitz optimisation
procedure.

\subsection{Linear optimisation based on sampling points}

Here, we first randomly select $N$ transformation parameters $\kappa_i \in
\mathcal{B}$, $i = 1, \ldots, N$, to obtain a sampled version of~\eqref{eq:bound_integral_optimisation} as follows
\begin{multline}\label{eq:bound_sample_optimisation}
       \min_{\underline{w}_j, \underline{b}_j, j = 1, \ldots, q} \frac{1}{N}\sum_{i=1}^N \left(\mathcal{G}_{u,v}(\kappa_i) -\big(\max_{j = 1,\ldots,q} \{\underline{w}_j^\tr \kappa_i + \underline{b}_j\}\big)\right) \\
        \text{subject to} \max_{j = 1,\ldots,q} \{\underline{w}_j^\tr \kappa_i + \underline{b}_j\} \leq \mathcal{G}_{u,v}(\kappa_i), i = 1,\! \ldots,\! N.
\end{multline}
We denote the optimal cost value of~\eqref{eq:bound_sample_optimisation} as $\beta^*$.
In~\eqref{eq:bound_sample_optimisation}, the number of piecewise
linear segments $q$ is fixed \emph{a priori}.
Still, problem~\eqref{eq:bound_sample_optimisation} is nontrivial to solve
jointly for all piecewise segments $\underline{w}_j, \underline{b}_j,
j = 1, \ldots, q$ unless $q = 1$
(where~\eqref{eq:bound_sample_optimisation} is reduced to a single
linear program). One difficulty is to determine the effective domain
of each piecewise linear segment.

To alleviate this, we propose to split
the whole domain $\mathcal{B}$ into $q$ sub-domains $\mathcal{B}_1,
\ldots, \mathcal{B}_q$, and then optimize each piecewise linear
segment over $\mathcal{B}_j$, $j = 1, \ldots, q$, individually. 
We then use
the following $q$ independent linear programs to approximate the
solution to~\eqref{eq:bound_sample_optimisation}:
\begin{equation}\label{eq:bound_sample_optimisation_individual}
    \begin{aligned}
      \beta^*_j :=  \min_{\underline{w}_j, \underline{b}_j} \;\; & \frac{1}{N}\sum_{\kappa_i \in \mathcal{B}_j }\left(\mathcal{G}_{u,v}(\kappa_i) - \big(\underline{w}_j^\tr \kappa_i + \underline{b}_j\big) \right) \\
        \text{subject to} \;\; &  \underline{w}_j^\tr \kappa_i + \underline{b}_j \leq \mathcal{G}_{u,v}(\kappa_i), \; i = 1, \ldots, N,
    \end{aligned}
\end{equation}
for $j = 1, \ldots, q$.  Note that in
\eqref{eq:bound_sample_optimisation_individual}, we minimise the
approximation error over only the sample points within a given domain
$\mathcal{B}_j$; however, we force each segment to satisfy the constraints at every sample point $\kappa_i \in \mathcal{B}$ over the whole domain. 

We have the following result for the quality of the solution from~\eqref{eq:bound_sample_optimisation_individual}.

\begin{proposition} \label{prop:linear_optimisation}
	Given any subdomains $\mathcal{B}_j$, $j = 1, \ldots, q$, the
	optimal solutions $\underline{w}_j, \underline{b}_j$, $j = 1,
	\ldots, q$, to~\eqref{eq:bound_sample_optimisation_individual}  are
	suboptimal to~\eqref{eq:bound_sample_optimisation}, i.e.,
	$\sum_{j=1}^q \beta^*_j \geq \beta^*$. There exists a set of
	subdomains $\mathcal{B}_j$, $j = 1, \ldots, q$, such that the optimal
	solutions to~\eqref{eq:bound_sample_optimisation}
	and~\eqref{eq:bound_sample_optimisation_individual} are identical,
	i.e.,  $\sum_{j=1}^q \beta^*_j = \beta^*$.
\end{proposition}
\begin{proof}
Consider the piecewise linear function in the objective function~\eqref{eq:bound_sample_optimisation}. Let $\mathcal{B}_j, j = 1, \ldots, q$ be the effective piecewise domain of the $j$th segment, \textit{i.e.},
\begin{equation} \label{eq:piecewise_domain}
    \max_{j = 1,\ldots,q} \{\underline{w}_j^\tr \kappa_i + \underline{b}_j\} = \begin{cases}
    \underline{w}_1^\tr \kappa_i + \underline{b}_1, & \text{if} \quad \kappa_i \in \mathcal{B}_1 \\
    \quad \vdots \\
    \underline{w}_q^\tr \kappa_i + \underline{b}_q, & \text{if} \quad \kappa_i \in \mathcal{B}_q.
    \end{cases}
\end{equation}
Then, the objective function~\eqref{eq:bound_sample_optimisation} can be equivalently written into
\begin{equation*}
\begin{aligned}
&\frac{1}{N}\sum_{i=1}^N \left(\mathcal{G}_{u,v}(\kappa_i) - \big(\max_{j = 1,\ldots,q} \{\underline{w}_j^\tr \kappa_i + \underline{b}_j\}\big)\right) \\
= &\frac{1}{N}\sum_{j=1}^q\left(\sum_{\kappa_i \in \mathcal{B}_j }\left(\mathcal{G}_{u,v}(\kappa_i) - \big(\underline{w}_j^\tr \kappa_i + \underline{b}_j\big) \right)\right)
\end{aligned}
\end{equation*}
Therefore,~\eqref{eq:bound_sample_optimisation} is equivalent to
\begin{multline}\label{eq:bound_sample_optimisation_s2_nonum}
      \min_{\underline{w}_j, \underline{b}_j, \mathcal{B}_j, j = 1, \ldots, q} \\
      \sum_{j=1}^q\left(\frac{1}{N}\sum_{\kappa_i \in \mathcal{B}_j }\left(\mathcal{G}_{u,v}(\kappa_i) - \big(\underline{w}_j^\tr \kappa_i + \underline{b}_j\big) \right)\right) \\
        \textbf{s.t.} \quad \underline{w}_j^\tr \kappa_i + \underline{b}_j \leq \mathcal{G}_{u,v}(\kappa_i), 
         \quad i = 1, \ldots, N, j = 1, \ldots q.
\end{multline}
Note that the piecewise domains $\mathcal{B}_j$ are determined by the linear segments $\underline{w}_j, \underline{b}_j, j = 1, \ldots, q$ implicitly in~\eqref{eq:piecewise_domain}. We need to simultaneously optimize the choices of $\mathcal{B}_j$ in~\eqref{eq:bound_sample_optimisation_s2}, making it computationally hard to solve.

A suboptimal solution for~\eqref{eq:bound_sample_optimisation_s2} is to \emph{a priori} fix the effective domain $\mathcal{B}_j$ and optimize over $\underline{w}_j, \underline{b}_j, j = 1, \ldots, q$ only, i.e.,
\begin{multline}\label{eq:bound_sample_optimisation_s2}
       \hat{\beta} := \\
       \min_{\underline{w}_j, \underline{b}_j, j = 1, \ldots, q} \sum_{j=1}^q\left(\frac{1}{N}\sum_{\kappa_i \in \mathcal{B}_j }\left(\mathcal{G}_{u,v}(\kappa_i) - \big(\underline{w}_j^\tr \kappa_i + \underline{b}_j\big) \right)\right) \\
        \textbf{s.t.} \quad \underline{w}_j^\tr \kappa_i + \underline{b}_j \leq \mathcal{G}_{u,v}(\kappa_i), 
         \quad i = 1, \ldots, N, j = 1, \ldots q,
\end{multline}
which is decoupled into $q$ individually linear programs, $j = 1, \ldots, q$
\begin{equation}\label{eq:bound_sample_optimisation_individual_s2}
    \begin{aligned}
      \beta^*_j :=  \min_{\underline{w}_j, \underline{b}_j} \quad & \frac{1}{N}\sum_{\kappa_i \in \mathcal{B}_j }\left(\mathcal{G}_{u,v}(\kappa_i) - \big(\underline{w}_j^\tr \kappa_i + \underline{b}_j\big) \right) \\
        \text{subject to} \quad &  \underline{w}_j^\tr \kappa_i + \underline{b}_j \leq \mathcal{G}_{u,v}(\kappa_i), \; i = 1, \ldots, N.
    \end{aligned}
\end{equation}
Therefore, it is clear that $\hat{\beta} = \sum_{j=1}^q \beta^*_j \geq \beta^*$. On the other hand, suppose the optimal solution to~\eqref{eq:bound_sample_optimisation} leads to the optimal effective domains $\mathcal{B}_j,j = 1, \ldots, q$ in~\eqref{eq:piecewise_domain}. Then, using this set $\mathcal{B}_j,j = 1, \ldots, q$, the decoupled linear programs~\eqref{eq:bound_sample_optimisation_individual_s2} are equivalent to~\eqref{eq:bound_sample_optimisation_s2} and~\eqref{eq:bound_sample_optimisation}. 
\end{proof}
To obtain a good solution~\eqref{eq:bound_sample_optimisation},
choosing the subdomains $\mathcal{B}_j$ becomes essential. A uniform
grid partition is one, naive choice. Another is to partition the
subdomains based on the distribution of the sampling points
$\mathcal{G}_{u,v}(\kappa_i)$.
The details of the splitting procedure are provided in the appendix. 

%

\begin{remark} \normalfont \textbf{(Explicit input splitting vs. piecewise linear constraints)}
    We note that one can perform explicit input splitting
	$\mathcal{B}_j, j = 1, \ldots, q, $ and verify each of them by
	solving~\eqref{eq:nn_verification} separately in order to certify
	the original large domain $\mathcal{B}$. The main drawback of this
	explicit input splitting is that we need to call a verifier for
	each subdomain $\mathcal{B}_j$ which can be hugely time consuming and not scalable. On the contrary, it only requires to solve multiple small linear
	programs~\eqref{eq:bound_sample_optimisation_individual} to derive
	our piece-wise linear constraints. Then, we only need to call a
	verifier once to solve the verification
	problem~\eqref{eq:nn_verification} over $\mathcal{B}$. 
	For tight verifiers, such as those mentioned in Section~\ref{sec:related-work}, this
	process is much more efficient than explicit input splitting.
\end{remark}

\subsection{Lipschitz optimisation for obtaining sound piecewise linear bounds}

The piecewise linear constraints
from~\eqref{eq:bound_sample_optimisation_individual} are valid for the
sampling points $\kappa_i \in \mathcal{B}, i = 1, \ldots, N$. To make
the constraints sound over all $\kappa \in \mathcal{B}$, we must shift
them such that all points on the pixel value function,
$\mathcal{G}_{u,v}(\kappa)$, satisfy the constraints in
\eqref{eq:pw_linear_bounds}. For this, we define a new function that
tracks the violation of a piecewise bound over the entire domain
$\mathcal{B}$:
\begin{equation}\label{eq:violation_function}
    \begin{aligned}
   \xi^*_{u,v} := \max_{\kappa\in \mathcal {B}} \quad
   \underline{f}_{u, v}(\kappa),
    \end{aligned}
\end{equation}
where $ \underline{f}_{u,
   v}(\kappa) = \max_{j = 1,\ldots,q} \{\underline{w}_j^\tr \kappa + \underline{b}_j\} - \mathcal{G}_{u,v}(\kappa).$ Then, we naturally have a sound piecewise linear lower bound as
$$
    \max_{j = 1,\ldots,q} \{\underline{w}_j^\tr \kappa + \underline{b}_j\} - \xi^*_{u,v} \leq \mathcal{G}_{u,v}(\kappa), \; \forall \kappa \in \mathcal{B}.
$$
However, computing the exact maximum $\xi^*$ is computationally hard
due to the \textit{nonconvexity, nonlinearity and nonsmoothness} of $\underline{f}_{u, v}(\kappa)$. Instead, given any $\epsilon > 0$, we can use a branch-and-bound Lipschitz optimisation procedure to find $\underline{\xi}^* \in \mathbb{R}$ satisfying
$
   \underline{\xi}^*\leq  \xi^*_{u,v} \leq \underline{\xi}^* + \epsilon.
$

 To establish the branch-and-bound Lipschitz optimisation procedure, we need to
  characterise the properties of the violation function $\underline{f}_{u,v}(\kappa)$.
\begin{proposition} \label{prop:lipschitz_constants}
   The violation function $\underline{f}_{u, v}(\kappa) := \max_{j =
   1,\ldots,q} \{\underline{w}_j^\tr \kappa + \underline{b}_j\} -
   \mathcal{G}_{u,v}(\kappa)$ is nonconvex, nonsmooth, and Lipschitz continuous
   over $\mathcal{B} \subset \mathbb{R}^d$. Furthermore, there exist
   $L_m > 0$, $m = 1, \ldots, d$, such that $\forall \kappa_1, \kappa_2 \in \mathcal{B}$
   \begin{equation} \label{eq:Lipschitz_bound}
      |\underline{f}_{u, v}(\kappa_1) - \underline{f}_{u, v}(\kappa_2)| \leq \sum_{m=1}^d L_m |\kappa_1(m) - \kappa_2(m)|.
   \end{equation}
\end{proposition}

\begin{proof}
The pixel value function is given by
$
    \mathcal{G}_{u,v}(\kappa) := \mathcal{P}_{\alpha, \beta} \circ \mathcal{I} \circ \mathcal{T}^{-1}_\mu(u,v).
$
We know that the spatial transformation $\mathcal{T}_{\mu}(u,v)$ and $\mathcal{P}_{\alpha, \beta}$ are continuous and differentiable everywhere. The interpolation function $\mathcal{I}(u,v)$ is continuous everywhere, but it is only differentiable within each interpolation region and it can be nonsmooth on the boundary. Also, $\mathcal{T}_{\mu}(u,v)$ and $\mathcal{I}(u,v)$ are generally nonconvex.

In addition, the piecewise linear function
$\max_{j =
   1,\ldots,q} \{\underline{w}_j^\tr \kappa + \underline{b}_j\}$
    is continuous but not differentiable everywhere. Therefore, the violation function $\underline{f}_{u, v}(\kappa)$ is nonconvex and nonsmooth in general. Finally, all the functions $\mathcal{T}_{\mu}(u,v)$, $\mathcal{P}_{\alpha, \beta}$, $\mathcal{I}(u,v)$ and $\max_{j =
   1,\ldots,q} \{\underline{w}_j^\tr \kappa + \underline{b}_j\}$ are Lipschitz continuous, so is the violation function $\underline{f}_{u, v}(\kappa)$. Thus, there exist
   $L_m > 0$, $m = 1, \ldots, d$, such that \eqref{eq:Lipschitz_bound} holds. 
\end{proof}
The properties of the violation function $\underline{f}_{u,v}(\kappa)$ in Proposition~\ref{prop:lipschitz_constants} are directly inherited from nonconvexity and nonsmoothness of the interpolation function $\mathcal{I}(u,v)$. The Lipschitz continuity is also from the interpolation function and the piecewise linear function.

With the information of $L_m$ in~\eqref{eq:Lipschitz_bound}, we are
ready to get a lower and an  upper bound  for  $\xi^*$ upon evaluating
the function at any point $\kappa_0 \in \mathcal{B}$:
\begin{equation} \label{eq:lower_upper_xi}
\begin{aligned}
   \underline{f}_{u, v}(\kappa_0)  &\leq \xi^* \\
   &= \max_{\kappa \in \mathcal{B}} \; \underline{f}_{u, v}(\kappa) \\
   &\leq  \max_{\kappa \in \mathcal{B}} \;  \underline{f}_{u, v}(\kappa_0) + \sum_{m=1}^d L_m |\kappa(m) - \kappa_0(m)| \\
   &\leq \underline{f}_{u, v}(\kappa_0) + \sum_{m=1}^d L_m h_m,
\end{aligned}
\end{equation}
where $h_m > 0$ denotes the difference of the lower and upper bound in each box constraint of $\mathcal{B}$. These lower and upper bounds~\eqref{eq:lower_upper_xi} are useful in the branch-and-bound procedure.

Still, we need estimate the Lipschitz constant $L_m$ in~\eqref{eq:Lipschitz_bound}. In our work, we show how to estimate the constant $L_m$ based on the gradient of $ \underline{f}_{u, v}(\kappa)$ whenever it is differentiable (note that $ \underline{f}_{u, v}(\kappa)$ is not differentiable everywhere)

\begin{proposition} \label{proposition:gradient}
    Let $\mathrm{Diff}(\mathcal{B})$ be the subset of $\mathcal{B}$
	where $\underline{f}_{u, v}(\kappa)$ is differentiable. Then, the Lipschitz constants in~\eqref{eq:Lipschitz_bound} can be chosen as
    $
       L_m = \sup_{\kappa \in \mathrm{Diff}(\mathcal{B})} |\nabla \underline{f}_{u, v}^\tr e_m|,
    $
    where $e_m \in \mathbb{R}^d$ is a basis vector with only the $m$-th element being one and the rest being zero.
\end{proposition}

\begin{proof}
This proof is motivated by~\cite{jordan2020exactly}. In order to prove Proposition~\ref{proposition:gradient}, we first state a useful result from~\cite[Lemma 3]{jordan2020exactly}. Let $f:\mathbb{R}^n \rightarrow \mathbb{R}$ be Lipschitz continuous over an open set $\Omega \subset \mathbb{R}^n$. We denote $\mathrm{Diff}(\Omega)$ as the subset of $\Omega$ where $f(x)$ is differentiable. We also let $\mathcal{D}$ be the set of $(x,v) \in \mathbb{R}^{2n}$ for which the directional derivative, $\nabla_v f(x)$, exists and $x \in \Omega$. Finally, we let $\mathcal{D}_v$ be the set
$
    \mathcal{D}_v = \{x \in \mathbb{R}^n\mid (x,v) \in \mathcal{D}\}.
$
Then, we have the following inequality~\cite[Lemma 3]{jordan2020exactly}
\begin{equation} \label{eq:derivative_inequality}
    \sup_{x \in \mathcal{D}_v} |\nabla_v f(x)| \leq \sup_{x \in \mathrm{Diff}(\Omega)} |\nabla f(x)^\tr v|.
\end{equation}

We now proceed to prove Proposition~\ref{proposition:gradient}.
Fix any $\kappa_1, \kappa_2 \in \mathcal{B}$, and we define a function $h: \mathbb{R} \rightarrow \mathbb{R}$ as
$
    h(t) = \underline{f}_{u, v}(\kappa_1 + t(\kappa_2-\kappa_1)).
$
Since $\underline{f}_{u, v}(\kappa)$ is Lipschitz continuous in $\mathcal{B}$, it is clear that $h(t)$ is Lipschitz continuous on the interval $[0,1]$. Thus, by Rademacher's Theorem, $h(t)$ is differentiable everywhere except for a set of measure zero.

We can further define a Lebesgue integrable function $g(t)$ that equal to $h'(t)$ almost everywhere as follows
$$
    g(t) = \begin{cases} h'(t), & \text{if} \quad h'(t) \quad \text{exists} \\
    \sup_{s \in [0,1]} |h'(s)|, &\text{otherwise}
    \end{cases}.
$$
Note that if  $\underline{f}_{u, v}(\kappa)$ is differentiable at some point, we have
$$
    h'(t) = \nabla \underline{f}_{u, v}(\kappa_1 + t(\kappa_2 - \kappa_1))^\tr (\kappa_2 - \kappa_1).
$$
Then we have the following inequalities
\begin{equation*}
\begin{split}
    & |\underline{f}_{u, v}(\kappa_1) - \underline{f}_{u, v}(\kappa_2)| = |h(1) - h(0)| = \left|\int_0^1 g(t) dt\right| \\
    & \leq \int_0^1 |g(t)| dt \\
    & \leq \int_0^1 \sup_{s \in [0,1]} |h'(s)| dt = \sup_{s \in [0,1]} |h'(s)| \\
    & \leq \sup_{\kappa \in \mathcal{D}_{\kappa_2-\kappa_1}} |\nabla_{\kappa_2-\kappa_1} \underline{f}_{u, v}(\kappa)|.
\end{split}
\end{equation*}
Furthermore, considering the inequality in~\eqref{eq:derivative_inequality}~\cite[Lemma 3]{jordan2020exactly}, we have
\begin{equation*}
\begin{split}
    & |\underline{f}_{u, v}(\kappa_1) - \underline{f}_{u, v}(\kappa_2)| \leq \sup_{\kappa \in \mathrm{Diff}(\mathcal{B})} |\nabla \underline{f}_{u, v}(\kappa)^\tr  (\kappa_2 - \kappa_1)| \\
    & \leq \sum_{m=1}^d \sup_{\kappa \in \mathrm{Diff}(\mathcal{B})} |\nabla \underline{f}_{u, v}(\kappa)^\tr e_m|  |\kappa_2(m) - \kappa_1(m)|
\end{split}
\end{equation*}
where $e_m \in \mathbb{R}^d$ is a basis vector with only the $m$-th element being one and the rest being zero. Therefore, the Lipschitz constants in~\eqref{eq:Lipschitz_bound} can be chosen as
    $
       L_m = \sup_{\kappa \in \mathrm{Diff}(\mathcal{B})} |\nabla \underline{f}_{u, v}(\kappa)^\tr e_m|.
    $
\end{proof}






\noindent\textbf{Maximum directional gradient.} To bound the maximum violation $ \xi^*_{u,v}$ in \eqref{eq:violation_function} using~\eqref{eq:lower_upper_xi}, we need to estimate the constant $L_m$, and Proposition~\ref{proposition:gradient} requires us to calculate the maximum directional gradient $|\nabla \underline{f}_{u, v}^\tr e_m|$. Each component of $\nabla \underline{f}_{u, v}$ varies independently with respect to any constituent of the transformation composition, $\mathcal{T}_{\mu}(\kappa_m)$, $m=1, \ldots, d$. Each $L_m$ depends only on a transformation, $\mathcal{T}_\mu$, and interpolation, $\mathcal{I}_{u,v}$. 
The only component that is not differentiable everywhere in the parameter space $\kappa \in \mathcal{B}$, is interpolation $\mathcal{I}_{u,v}(x,y)$ - this due to it being disjoint across interpolation regions. We overcome this by calculating the interpolation gradient, $\nabla_{x,y} \mathcal{I}_{u,v}$ separately in each interpolation region, and taking the maximum interval of gradients from the union, $[\nabla_{x,y}I_{min},\nabla_{x,y}I_{max}]  = [\min(\cup_{k=1, \ldots, n}\nabla_{x,y}I_{k}), \max(\cup_{k=1, \ldots, n}\nabla_{x,y}I_{k})]$, where $I_k$ are the relevant interpolation regions, and $\cup_{k=1,\ldots,n} I_{k} = \mathcal{R} \subset \mathcal{B}$. Computing a bound on $L_m$ this way mirrors the IBP-based procedure outlined in~\cite{balunovic2019certifying}. With this we can calculate an upper bound on $L_m$ to be applied in the Lipschitz algorithm.

\noindent\textbf{Branch-and-bound Lipschitz optimisation procedure.} Similar to~\cite{balunovic2019certifying}, we use a branch-and-bound procedure (See Appendix) where $\underline{f}_{u,v}$ and $\mathcal{B}$ are given as inputs alongside the Lipschitz error, $\epsilon$, and samples per subdomain, $n$. The procedure first samples the violation function $\underline{f}_{u,v}$, obtaining maximum value candidates, this is placed in a list of 3-tuples with the upper bound, $\underline{f}_{{\rm bound},i}$, and corresponding domain, $\mathcal{B}_i$. The key upper bound operation $\rm bound(\cdot)$ is obtained using~\eqref{eq:lower_upper_xi}.
We then check whether each 3-tuple in our list meets the termination criteria, as parameterised by $\epsilon$. If the requirement is satisfied for all elements then we terminate and return $\underline{\xi}^*$. Until the requirement is met for every list element we iteratively split unsatisfied subdomains. This process is repeated until a satisfactory maximum candidate is found, splitting $\kappa$ in each iteration. We can ignore any sub-domain, $\kappa_{\mathfrak{1}}^n$, of $\kappa_{\mathfrak{1}}$ where the function bound $\underline{f}_{\text{bound}}$ in $\kappa_{\mathfrak{1}}^n$ is smaller than a maximum value candidate $\underline{f}_{\text{max}}$ in any other sub-domain. Deciding how to split subdomains is non-trivial for higher dimensional parameter spaces. In the case $\kappa \in \mathbb{R}^1$ we need only decide where to split on a single axis; for which we use the domain midpoint. The crux of our algorithm is approximating the gradient of $\underline{f}_{u,v}(\kappa)$ when it is differentiable, as stated in Proposition~\ref{proposition:gradient} (see appendix for further details on the branch-and-bound procedure). For bounding the violation of piecewise linear bounds we can consider the piecewise bound itself to be made of $q$ linear sub-regions with each one bounded by the intersection with the neighbouring linear piece - or the lower and upper bounds on the transformation parameters. We can then bound the Lipschitz constant in the same way as for a single linear bound, instead starting with $q$ sub-domains. Solving the Lipschitz bounding procedure for each linear segment over only its local domain in this way enables us to bound the Lipschitz constant of a piecewise linear bound in the same time as a linear bound takes.

\section{Experimental Evaluation} \label{section:experiments}
In this section we present three sets of results: (i) a quantitative
study directly comparing the model-agnostic bounds produced by our
piecewise linear approach against the state-of-the-art linear
bounds~\cite{balunovic2019certifying}, (ii) an empirical evaluation of
verification results obtained using linear and piecewise linear
bounds, without input splitting and using the same neural network
verifier~\cite{botoeva2020efficient}, and (iii) a comparison
of our results against the present
state-of-the-art method~\cite{balunovic2019certifying}.


\subsection{Experimental setup}  We consider the MNIST image recognition dataset~\cite{LecunCortesBurges98} and CIFAR10~\cite{Krizhevsky2009LearningML}.
In line with the previous literature~\cite{balunovic2019certifying},
we use two fully-connected ReLU networks, MLP2 and MLP6, and one
convolutional ReLU network, CONV, from the first competition for
neural network verification (VNN-COMP)~\cite{vnncomp}. The
fully-connected networks comprise 2 and 6 layers respectively. Each
layer of each of the networks has 256 ReLU nodes. The convolutional
network comprises two layers. The first layer has 32 filters of size
$5\times5$, a padding of 2 and strides of 2. The second layer has 64
filters of size of $4 \times 4$, a padding of 2 and strides of 1.
Additionally, we employ a larger convolutional ReLU network from
relevant previous work~\cite{balunovic2019certifying}, composed of
three layers: a convolutional layer with 32 filters of size $4\times4$
and strides of 2, a convolutional layer with 64 filters of size
$4\times4$ and strides of 2, and a fully connected layer with 200
nodes. 
All experiments were carried out on an Intel Core
i9-10940X (3.30GHz, 28 cores) equipped with 256GB RAM and running
Linux kernel~5.12. 
DeepG experiment ANS: we do not use GPU in these experiments.

Once a convex over-approximation of the attack space $\Omega_{\epsilon}(\bar{x})$ is computed, (cf. Section~\ref{section:geometric_robustness}) a neural network verifier is required to provide a lower bound on problem~\eqref{eq:nn_verification}.
Unless stated otherwise, the verification results reported in this work are obtained using \texttt{VENUS}, a complete MILP-based verification toolkit for feed-forward neural networks~\cite{botoeva2020efficient}. 


\begin{figure}[t!]
	\centering
	\includegraphics[scale=0.075]{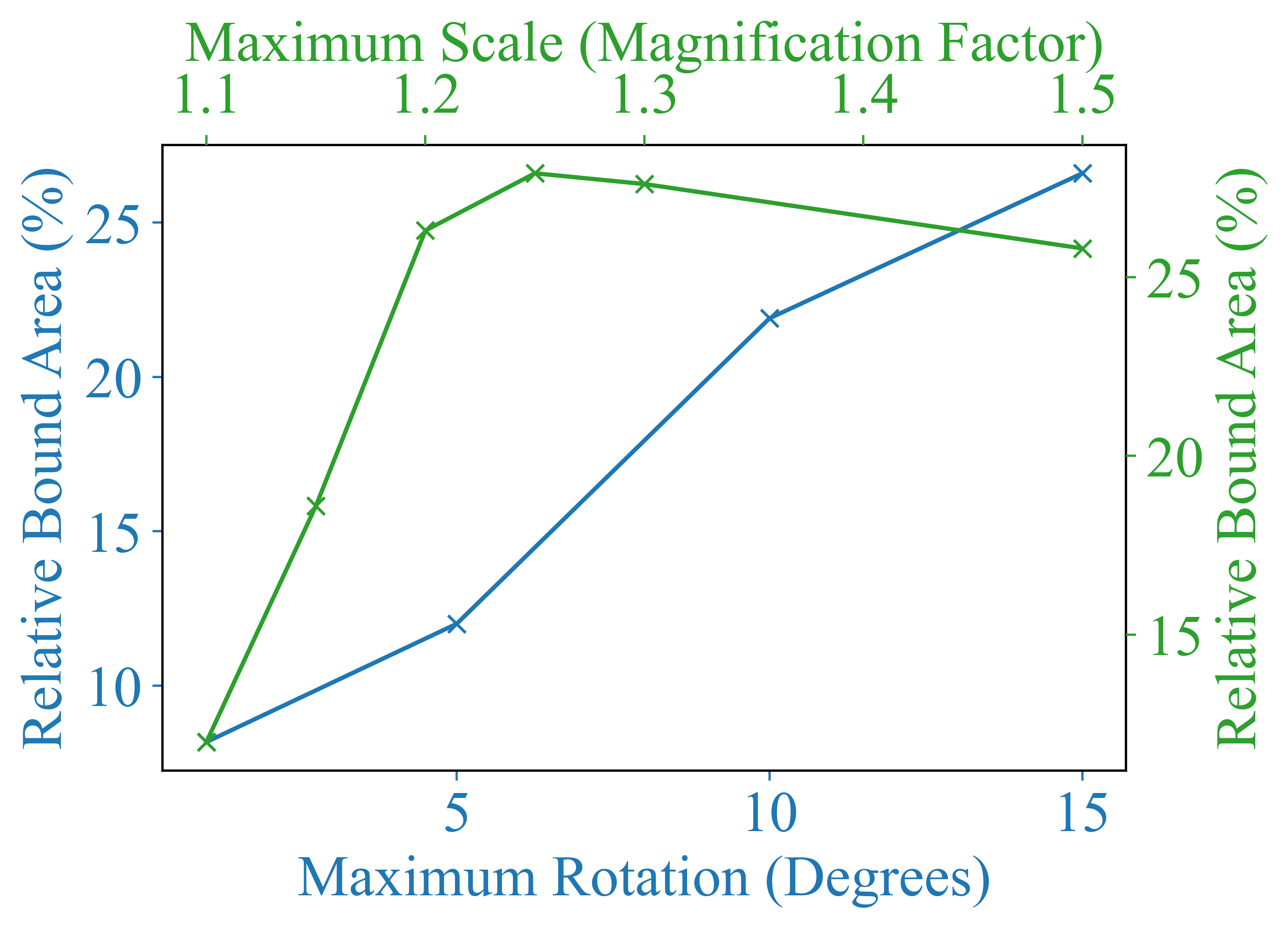}
	\includegraphics[scale=0.075]{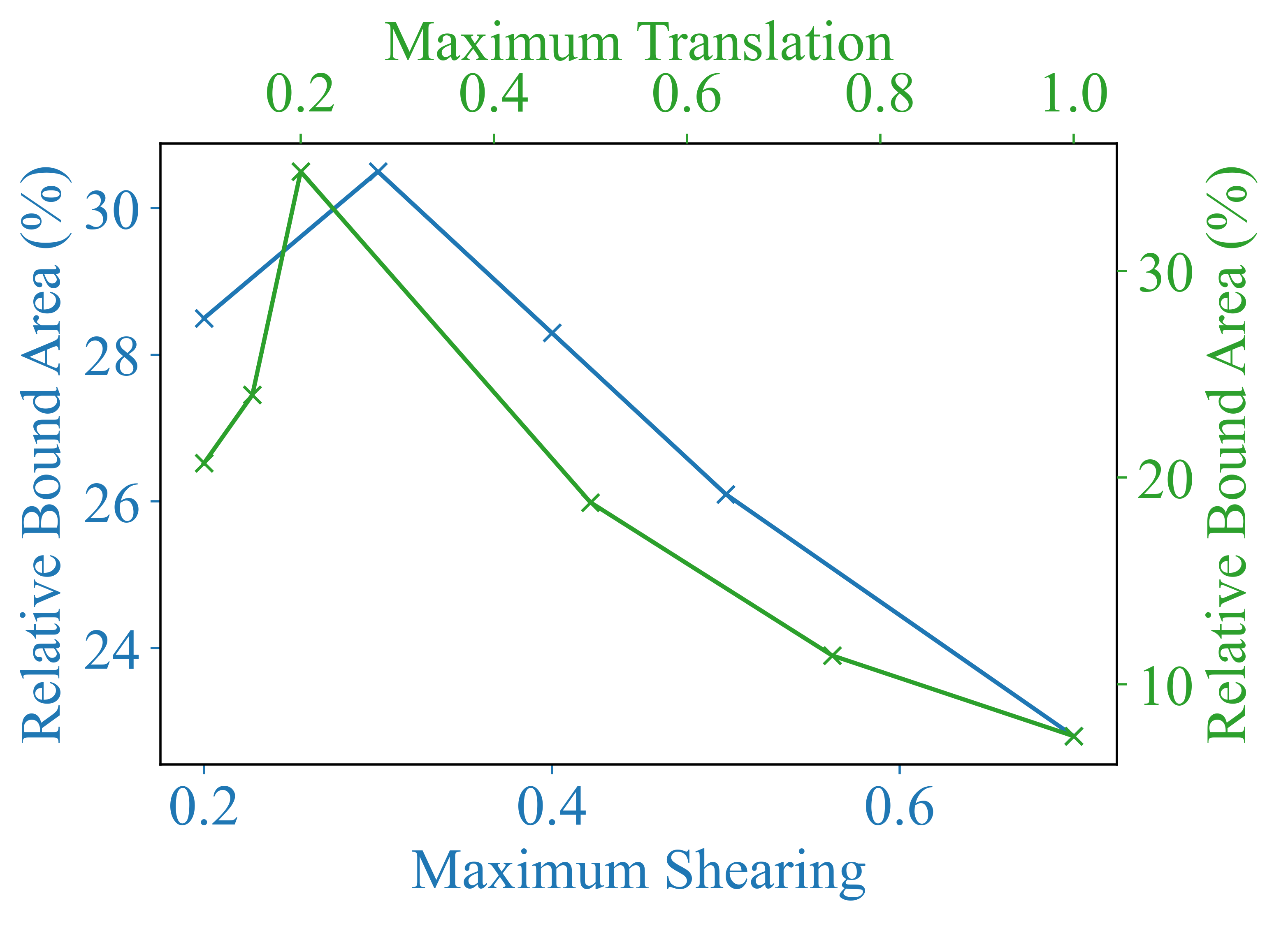}
		\caption{A comparison of area captured by piecewise linear and linear bounds as a function of transformation parameter. Relative bound area is defined as $1 - (V_{\text{PWL}} / V_\text{L})$.}
	\label{fig:area_comparison}
\end{figure}

\begin{table}[t]
	\setlength{\tabcolsep}{4.8pt}
	\caption{Comparison of verification results for piecewise linear
	constraints and linear constraints.}
	\label{table:exp}
	\centering
\begin{tabular}{c|c|cc|ll|ll}
\toprule
\multirow{2}{*}{Model}                   & \multirow{2}{*}{Attack} & \multicolumn{2}{c|}{Verified}                         & \multicolumn{2}{c|}{Falsified} & \multicolumn{2}{c}{Time (s)} \\ \cline{3-8}
                                           &                                 & L                   & PWL                   & L        & PWL        & L         & PWL         \\ \hline
\multirow{4}{*}{MLP2}                    & R(5)                               & 26                     &                      \textbf{28} &     74          &      72      &      0.7          &       3.7      \\
                                           & Sh(0.2)                              & 20                     &           \textbf{26}            &       80        &     74       &        1.2        &            12 \\
                                           & Sc(1.1)                              &             24             &                      24 &       76        &       76     &          1.1      &      51       \\
                                           & T(0.1)                               & 16     & 16 &        84       &           84 &           11     &         54    \\ \hline
\multirow{4}{*}{MLP6}                    & R(15)                               & 0                     &                      \textbf{2} &        12       &      32      &        1602        &    1253         \\
                                           & Sh(0.5)                              & 0                     &                      0 &        16       &       \textbf{68}     &       1591         &      778       \\
                                           & Sc(1.3)                              &          0                &                      0 &         24      &       \textbf{78}     &        1404        &       727      \\
                                           & T(0.2)                               & 0     & \textbf{2} &        26       &           74 &         1397       &       648      \\ \hline
\multicolumn{1}{l|}{\multirow{4}{*}{CONV}} & R(10)                               & \multicolumn{1}{l}{20} & \multicolumn{1}{l|}{\textbf{48}} &        2       &    0        &         1447       &           1044 \\
\multicolumn{1}{l|}{}                      & Sh(0.2)                              & \multicolumn{1}{l}{18} & \multicolumn{1}{l|}{\textbf{50}} &       0        &       0     &          1548      &            1044 \\
\multicolumn{1}{l|}{}                      & Sc(1.3)                              & \multicolumn{1}{l}{0}     & \multicolumn{1}{l|}{\textbf{10}} &         4      &       4     &        1750        &  1663           \\
\multicolumn{1}{l|}{}                      & T(0.15)                               & \multicolumn{1}{l}{0}     & \multicolumn{1}{l|}{\textbf{32}} &        2       &      0      &        1767        &  1397 \\
\bottomrule 
\end{tabular}
\end{table}

\subsection{Experimental results} 

In the following, we will use ``L" to denote the linear relaxation from equation~\eqref{eq:linear_bounds}, and ``PWL" to denote the piecewise linear relaxation from equation~\eqref{eq:pw_linear_bounds}.

\begin{table*}[t!]
\caption{Comparison of L and PWL using \texttt{VENUS}, with verification results taken from DeepG~\cite{balunovic2019certifying}.}
\label{tab:comparison deepg}
\centering
\begin{tabular}{ccccclcl}
\multicolumn{1}{c|}{\multirow{2}{*}{Dataset}} & \multicolumn{1}{c|}{\multirow{2}{*}{Transformation}} & \multicolumn{1}{c|}{\multirow{2}{*}{Accuracy (\%)}} & \multicolumn{1}{c|}{DeepG}          & \multicolumn{2}{c|}{Linear (Ours)}             & \multicolumn{2}{c}{PWL (Ours)} \\ \cline{4-8} 
\multicolumn{1}{c|}{}                         & \multicolumn{1}{c|}{}                                & \multicolumn{1}{c|}{}                               & \multicolumn{1}{c|}{Certified (\%)} & Certified (\%) & \multicolumn{1}{l|}{Time (s)} & Certified (\%)    & Time (s)   \\ \hline
MNIST                                         & R(30)                                                & 99.1                                                & 87.8                                & 90.8           & 37.9                          & 92.9              & 28.3       \\
CIFAR                                         & R(2)Sh(2)                                            & 68.5                                                & 54.2                                & 65.0           & 239.5                         & 66.0              & 204.9     
\end{tabular}%
\end{table*}

\paragraph{PWL vs L: comparing areas.}
Figure~\ref{fig:area_comparison} is a direct comparison of bound \textit{tightness} between our piecewise linear bounds and the current state-of-the-art linear bounds~\cite{balunovic2019certifying}. For each image, linear and piecewise linear bounds are generated, each one capturing the reachable pixel values for a given transformation. We always use two piecewise segments ($q=2$) and use a Lipschitz error of 0.01 to compute bounds.
The area enclosed by each set of bounds is then calculated and averaged for every pixel over all images. In each case the piecewise linear bounds are guaranteed to be tighter (enclose a smaller area) than the linear bounds, as in Section~\ref{section:pw_linar_bounds}.
Figure~\ref{fig:area_comparison} shows the relative area
(specifically, $1 - (V_{\text{PWL}} / V_\text{L})$ with
$V_{\text{PWL}}$ and $V_{\text{L}}$ being the volume enclosed by the
piecewise linear and linear bounds, respectively) of the two bound
types. In Figure~\ref{fig:area_comparison}, there is an initial
increase in relative tightness for all transformations – this is a
result of linear bounds being unable to efficiently capture the
increasing nonlinearity in the pixel value curve,
$\mathcal{G}_{u,v}(\kappa)$. After an initial increase, the behaviour for different transformations diverges. For rotation, the relative advantage of the piecewise bounds continues to increase up to 15 degrees. For scaling, however, there is a peak at 1.25$\times$ magnification, followed by a decrease in the relative tightness. This result is explained by a corresponding increase in the complexity of the pixel value curve. 
Notably, the piecewise bounds are best suited to nonmonotonic pixel
value curves with a single, sharp vertex. For curves with many
vertices and large fluctuations, piecewise linear bounds become
increasingly linear (the gradient of the pieces converge) to maintain
convexity. Though this is the case for $q=2$, as we study here, for
larger numbers of piecewise segments the advantage over linear bounds
will continue to hold, as the piecewise bounds approximate the convex
hull of the pixel values for $q \rightarrow \infty$.
The plots for shearing and translation show a similar pattern to scaling. Although the relative tightness may decrease for larger transformations, the total bounded area increases, making any proportional reduction in area more significant.

\paragraph{PWL vs L: verification results.}
Table~\ref{table:exp} reports the experimental results obtained for verification queries using \texttt{VENUS}, on the VNN-COMP networks. For each type of input bound – piecewise linear and linear – the table shows the percentage of certified images (Verified column), the percentage of images for which a valid counter example was found (Falsified column), and the average verification time. 
We verify the robustness of each of the networks with respect to one of four transformations - rotation, scaling, shearing, or translation - on 50 randomly selected images from the MNIST test set. For each verification query we use a timeout of 30 minutes. 
We observe a considerable performance advantage using piecewise linear bounds for the convolution network, in every case, at least doubling the count of verifiable images.
For the 6-layer MLP network, many of the transformations tried could not be verified, leading to numerous counter examples and time-outs. However, for every transformation the piecewise linear bounds were able to find more counter examples than linear bounds – this is a result of the improved tightness of piecewise linear bounds. For the 2-layer MLP, results across the bound types are very similar, in some cases they are equal. This is due to two factors, both of which stem from the network’s small size. Firstly, the 2-layer network is the least robust of all three. Accordingly, our results are for very small transformations for which the pixel value curve is approximately linear. In these cases, linear bounds can capture the input set as well as piecewise linear bounds. Secondly, the advantage of piecewise linear bounds' tightness is compounded over each layer of a network – the 2-layer MLP is so small that this effect is minimal, further aligning the performance of the approaches.
Finally, the use of piecewise linear constraints result in a reduction of average verification times on both the 6-layer MLP and the convolutional network: this is due to the fact that their relative tightness compensates for the additional cost of their encoding, leading the employed MILP-based verifier to positive lower bounds on the verification problem \eqref{eq:nn_verification} in less time.

\paragraph{Comparison with literature results.}

In Table~\ref{tab:comparison deepg} we provide a comparison of verification results obtained using \texttt{VENUS} with both linear and piecewise linear constraints, with the DeepG~\cite{balunovic2019certifying} results, obtained using linear constraints and the DeepPoly~\cite{singh2019abstract} verifier, which relies on a relatively loose LP relaxation of~\eqref{eq:nn_verification}. Further, we use a MILP-based verifier which enabled us to add the pixel domain constraints in addition to our transformation-based bounds. This, coupled with the tighter verifier, enables our linear bounds to out-perform those from DeepG.
We consider MNIST and CIFAR10 benchmark presented in~\citet{balunovic2019certifying}. The MNIST example consists of verifying a 30 degree rotation transformation by way of 10, 3-degree sub-problems. 
This is in contrast to Table~\ref{table:exp}, where each perturbation is represented by a single set of bounds and a single verifier call per image. 
Table~\ref{tab:comparison deepg} shows that, even under the small-perturbation setting, the use of tighter verification algorithms (L versus DeepG) increases the number of verified properties.
Furthermore, we show that the method proposed in this work, PWL, leads to the tightest certification results. 
The CIFAR10 example comprises a composition of rotation and shearing of 2 degrees and 2\% respectively. This query is solved via 4 sub-problems (with each transformation domain split in half). The results show a 12\% improvement for the PWL bounds over the DeepG result. However, much of this gain comes from the verifier itself. The gap between the linear bounds and their piecewise counterpart is 1\%. We attribute this smaller gap to the relatively small domain over which each sub-problem runs.
Nevertheless, we point out that verifying perturbations through a series of sub-problems is extremely expensive, as it requires repeated calls to both neural network verifiers, and to the constraint-generation procedure (including the branch-and-bound-based Lipschitz optimisation).
For this reason, we focus on verification the setting without transformation splitting, and aim to maximize certifications through the use of tight verifiers and over-approximations of the geometric transforms.





\section{Conclusions} \label{section:conclusions}

We have introduced a new piecewise linear approximation method for
geometric robustness verification. Our approach can generate 
provably tighter convex relaxations for images obtained by geometric
transformations than the state-of-the-art  
methods~\cite{singh2019abstract,balunovic2019certifying}. Indeed, we have shown experimentally that
the proposed method can provide better verification precision in
certifying robustness against geometric transformations than prior
work~\cite{balunovic2019certifying}, while being more computational
efficient. 

Despite the positive results brought by our piecewise linear approximation method, further topics deserve further exploration. Firstly, it remains challenging to obtain the optimal piecewise linear constraints via~\eqref{eq:bound_sample_optimisation}. To get a good set of piecewise linear constraints, our current method~\eqref{eq:bound_sample_optimisation_individual} requires to obtain a good heuristic partition of the domain $\mathcal{B}_1, \ldots, \mathcal{B}_q$. It will be interesting to further investigate and quantify the suboptimality of the solution from~\eqref{eq:bound_sample_optimisation_individual}. Second, the number of piecewise linear segment $q$ is a hyperparameter in our framework. A larger value $q$ leads to a better approximation of the pixel value function in theory; however, this also results in more linear constraints for the verification problem in practice. Future work will investigate how to choose a good value of $q$ based on the curvature of of the pixel value function. 


\section*{Acknowledgements}
Ben Batten was funded by the UKRI Centre for Doctoral Training in Safe
and Trusted Artificial Intelligence. Alessandro De Palma was supported
by the “SAIF" project, funded by the “France 2030” government
investment plan managed by the French National Research Agency, under
the reference ANR-23-PEIA-0006. Alessio Lomuscio was supported by a
Royal Academy of Engineering Chair in Emerging Technologies.


\bibliography{references}

\newpage
\appendix
\section*{Appendix}
\section{Linear optimisation over sub-domains}

Our discussion below focuses on $q = 2$ in~\eqref{eq:bound_sample_optimisation}, and with this choice, we have already found promising improvements in our experiments (see the main text). With this constraint we can find suboptimal piecewise bounds by solving two independent linear optimisation problems, where each problem is applied over a subset of the piecewise domain, divided at a given sample point, $n$. We name the parameter sub-spaces divided by $n$, $\boldsymbol{\kappa}_{\mathfrak{1}}$, and $\boldsymbol{\kappa}_{\mathfrak{2}}$ where $\boldsymbol{\kappa}_{\mathfrak{1}}, \boldsymbol{\kappa}_{\mathfrak{2}} \subset \mathcal{B}$. Expressing~\eqref{eq:bound_sample_optimisation} in this way gives
\begin{subequations} \label{eq:pw_bound_sample_optimisation}
\begin{align}
    \min_{\underline{w}_1, \underline{b}_1} \quad & \frac{1}{N}\sum_{i=1}^n \left(\mathcal{G}_{u,v}(\kappa_i) - \{\underline{w}_1^\tr \kappa_i + \underline{b}_1\}\right)  \label{eq:pw_bound_sample_optimisation_1}\\
    \text{subject to} \quad &  \{\underline{w}_1^\tr \kappa_i + \underline{b}_1\} \leq \mathcal{G}_{u,v}(\kappa_i), \quad i = 1, \ldots, N \nonumber, \\
    \min_{\underline{w}_2, \underline{b}_2} \quad & \frac{1}{N}\sum_{i=n}^N \left(\mathcal{G}_{u,v}(\kappa_i) - \{\underline{w}_2^\tr \kappa_i + \underline{b}_2\}\right) \label{eq:pw_bound_sample_optimisation_2} \\
    \text{subject to} \quad &  \{\underline{w}_2^\tr \kappa_i + \underline{b}_2\} \leq \mathcal{G}_{u,v}(\kappa_i), \quad i = 1, \ldots, N. \nonumber
\end{align}
\end{subequations}

In~\eqref{eq:pw_bound_sample_optimisation_1} and~\eqref{eq:pw_bound_sample_optimisation_2} we optimise the area over over only the sample points within a piece's domain, $\boldsymbol{\kappa}_{\mathfrak{1}}$, or $\boldsymbol{\kappa}_{\mathfrak{2}}$; however, we enforce the constraints at every sample point. By doing this we guarantee convexity of our piecewise constraints. We develop a heuristic to determine the sample point, $n$, at which we split $\mathcal{B}$ based on the error between sampled points and optimal \textit{linear} bounds

\begin{equation}\label{eq:splitting_heuristic}
    \begin{aligned}
    \underline{n} = \max_{i = 1, \ldots, N} \left(\{\overline{w}^\tr \kappa_i + \overline{b}\} - \mathcal{G}_{u, v}(\kappa_i) \right),
    \end{aligned}
\end{equation}
where $\underline{n}$ is the splitting point for the lower bound. We calculate $\overline{n}$ correspondingly using the lower linear bound. There exists a splitting point, $n$, that would produce optimal piecewise bounds, but finding it is infeasible. In practice, we first compute a single linear bound for lower and upper constraints and then use this bound to compute the splitting point from~\ref{eq:splitting_heuristic}. Then, once the piecewise bound is obtained, half of the original linear bound is effectively discarded for the verification procedure. We compute the bounds in this way for two reasons: firstly, it enables us to apply our splitting heuristic in~\ref{eq:splitting_heuristic}, and secondly, it is computationally efficient in our experimental setting where we require the linear bounds for comparison.

\section{Details of branch-and-bound procedure}

With our unsound constraints, our method closely follows that of~\cite{balunovic2019certifying}, with the important exception that we treat our single piecewise bound as two, separate linear bounds with domains, $\boldsymbol{\kappa}_{\mathfrak{1}}$, and $\boldsymbol{\kappa}_{\mathfrak{2}}$. We first define a function, $\underline{f}_{u,v}$, to track the violation of a bound by the pixel value function, $\mathcal{G}_{u,v}$. In the case that the lower bound is piecewise, we will maximise $\underline{f}_{u,v}$ twice over $\boldsymbol{\kappa}_{\mathfrak{1}}$, and $\boldsymbol{\kappa}_{\mathfrak{2}}$, and $\overline{f}_{u,v}$ once over $\mathcal{B}$. Maximisation of $\underline{f}_{u,v}$ is done via a branch-and-bound Lipschitz procedure. Algorithm~\ref{lipschitz_optimisation} shows a simplified version of the implementation we use. For each instance of $\underline{f}_{u,v}(\kappa) \text{ where } \kappa \in \boldsymbol{\kappa}_{\mathfrak{1}}$, we first approximate the Lipshitz constant, $L_i$, and use it to bound $\underline{f}_{u,v}$
\begin{equation}
    \begin{centering}
    \underline{f}_{\text{bound}} = L_i \frac{\boldsymbol{\kappa}_\mathfrak{1}}{2} + \left(\underline{f}_{u,v}(\kappa_i) \right),
    \end{centering}
\end{equation}
where $\kappa_{i}$ is the midpoint of $\boldsymbol{\kappa}_{\mathfrak{1}}$. We find upper bound candidates by sampling the violation function$f_{u,v}(\kappa)$ at four, evenly spaced points in $\boldsymbol{\kappa}_{\mathfrak{1}}$; the largest valued obtained becomes the maximum value candidate, $\underline{f}_{\text{max}}$. We aim to find a maximum value-bound pair that satisfies $\underline{f}_{\text{bound}} - \underline{f}_{\text{max}} < \epsilon$, with $\epsilon$ given. This process is repeated until a satisfactory maximum candidate is found, splitting $\boldsymbol{\kappa}$ in each iteration. We can ignore any sub-domain, $\boldsymbol{\kappa}_{\mathfrak{1}}^n$, of $\boldsymbol{\kappa}_{\mathfrak{1}}$ where the function bound $\underline{f}_{\text{bound}}$ in $\boldsymbol{\kappa}_{\mathfrak{1}}^n$ is smaller than a maximum value candidate $\underline{f}_{\text{max}}$ in any other sub-domain. This is because we can guarantee that the maximum value, in this case, is not in the $\boldsymbol{\kappa}_{\mathfrak{1}}^n$ sub-domain. We deal only with 1-dimensional parameter spaces for which we split at the midpoint. The outline of this procedure is given in Algorithm~\ref{lipschitz_optimisation}.

\begin{algorithm}[tb]
\caption{Branch-and-bound Lipschitz Optimisation Procedure}
\label{lipschitz_optimisation}
\textbf{Input}: $\underline{f}_{u, v}$, $\mathcal{B}$, $\epsilon$, $n$, $N$\\
\textbf{Output}: $\xi^*$
\begin{algorithmic}[1] 
\STATE $\underline{f}_{\text{max}} := \max_{l=1, \ldots, n} \underline{f}_{u,v}(\kappa_l)$, where $\kappa_l \in \mathcal{B}$.
\STATE $\underline{f}_{\text{bound}} := \text{bound}(\underline{f}_{u,v},\nabla \underline{f}_{u,v}, \mathcal{B})$, where the operation ${\rm bound}(\cdot)$ refers to~\eqref{eq:lower_upper_xi}.
\STATE $\mathcal{L}:= [(\underline{f}_{\max}, \underline{f}_{\text{bound}}, \mathcal{B})]$.
\WHILE{$\underline{f}_{\text{bound}, i=1,\ldots,N}-\underline{f}_{{\max}, i=1,\ldots,N} > \epsilon$}
\FOR{$i \leftarrow 1$ \textbf{to} $N$}
\IF {$\underline{f}_{\text{bound}, i}-\underline{f}_{{\max}, i} > \epsilon$}
\STATE $\mathcal{B}_{i, i+N} = \text{split}(\mathcal{B}_i)$.
\STATE $\underline{f}_{{\max}, i} := \max_{l=1, \ldots, n} \underline{f}_{u,v}(\kappa_l)$, where $\kappa_l \in \mathcal{B}_i$.
\STATE $\underline{f}_{\text{bound}, i} := \text{bound}(\underline{f}_{u,v}, \nabla \underline{f}_{u,v}, \mathcal{B}_i)$.
\STATE $\mathcal{L}_i := [(\underline{f}_{{\max}, i}, \underline{f}_{\text{bound}, i}, \mathcal{B}_i)]$.
\ENDIF
\ENDFOR
\ENDWHILE
\STATE \textbf{return} $\xi^* = \max_{i=1, \ldots, N}\underline{f}_{{\max}, i}$.
\end{algorithmic}
\end{algorithm}

\end{document}